%% file: SSVM.tex
\documentclass{llncs}

\usepackage{amsmath}
\usepackage{amsfonts}
\usepackage{amssymb}
\usepackage{graphicx} 
\usepackage{dcolumn} 
\usepackage{bm} 
\usepackage{subcaption}
\usepackage{color}
\usepackage{mathtools}
\usepackage{enumitem}

\mathtoolsset{showonlyrefs}

\newtheorem{Proposition}{Proposition}

\newcommand{\R}{\mathbb{R}}
\DeclareMathOperator{\diag}{diag}
\def\tT{{\mbox{\tiny{T}}}}
\renewcommand{\bullet}{\,\begin{picture}(-1,2)(-1,-2)\circle*{0.5}\end{picture}\ \,}

\title{Invertible Neural Networks versus MCMC for
Posterior Reconstruction in Grazing Incidence X-Ray Fluorescence} 

\author{Anna Andrle\inst{1} \and Nando Farchmin\inst{1} \and Paul Hagemann \inst{1,2} \and Sebastian Heidenreich \inst{1} \and Victor Soltwisch \inst{1} \and Gabriele Steidl \inst{2}}

\institute{Department of Mathematical Modelling and Data Analysis \& Department of Radiometry with Synchrotron Radiation, Physikalisch-Technische Bundesanstalt Braunschweig und Berlin, Abbestrasse 2-12, D-10587 Berlin, Germany, \email{p.hagemann@campus.tu-berlin.de} 
\and
Institute of Mathematics, TU Berlin, Straße des 17. Juni 136, D-10623 Berlin, Germany}

\begin{document}

\maketitle
\date{\today}


\begin{abstract}
 Grazing incidence X-ray fluorescence is a non-destructive technique  for
 analyzing the geometry and compositional parameters of nanostructures appearing e.g. in computer
 chips.  In this paper, we propose to reconstruct the posterior parameter
 distribution given a noisy measurement generated by the forward model by an
 appropriately learned invertible  neural  network.  This network resembles the
 transport map from a reference distribution  to the posterior. We demonstrate
 by numerical comparisons that  our method can compete with established Markov
 Chain Monte Carlo approaches, while being more efficient and flexible in
 applications.
 \end{abstract}




\input{src/intro.tex}

\input{src/math.tex}

\input{src/application.tex}

\input{src/results.tex}

\input{src/conclusion.tex}

\bibliographystyle{abbrv}
\bibliography{references}

\end{document}

%% file: src/intro.tex
\section{Introduction}
Computational progress is deeply tied with  making the structure of computer
chips smaller and smaller.  Hence there is a need for efficient methods that
investigate the critical dimensions of a microchip.  
Optical scattering techniques are frequently used for the characterization of
periodic nanostructures on surfaces in the semiconductor industry
\cite{henn2014improved,mack2008fundamental}. 
As a non-destructive technique, grazing
incidence X-ray fluorescence (GIXRF)  is of particular interest for many industrial
applications. 
Mathematically, the reconstruction of nanostructures, i.e., of their geometrical  parameters, 
can be rephrased in an inverse problem. Given grazing incidence X-ray fluorescence
measurements $y$, we want to recover the distribution of the parameters $x$ of a grating. 
To account for measurement errors, it appears to be crucial to take a Bayesian
perspective. The main cause of uncertainty is due to inexact measurements $y$,
which are assumed to be corrupted by additive Gaussian noise with different variance
in each component.

The standard approach to recover the distribution of the parameters are Markov Chain Monte Carlo 
(MCMC) based algorithms \cite{AFDJ2003}. 
Instead, in this paper, we make use of invertible
neural networks (INNs) \cite{dinh2017density,kruse2020hint} within the general concept of transport maps \cite{Marzouk_2016}. 
This means that we sample from a reference distribution and seek a diffeomorphic
transport map, or more precisely its approximation by an INN, 
which maps this reference distribution to the problem posterior. 
This approach has some advantages over standard MCMC--based methods:
i) Given a transport map, which is computed in an offline step, the generation of independent posterior samples essentially reduces to sampling the freely chosen reference distribution. Additionally, observations indicate that learning the
transport map requires less time than generating a sufficient amount of
independent samples via MCMC.  
ii) Although the transport map is conditioned on
a specific measurement, it can serve as a good initial guess for the transport related to similar measurements
or as a prior in related inversion problems.  Hence the effort to find a transport for different runs within the same experiment reduces drastically. 
An even more sophisticated way of using a pretrained diffeomorphism has been recently
suggested in \cite{siahkoohi2021preconditioned}.

Having trained the INN, we compare its ability to recover the posterior distribution
with the established MCMC method for fluorescence experiments.
Although, a similar INN approach with a slightly simpler noise model was recently also reported for reservoir characterization in 
\cite{rizzuti2020parameterizing},
we are not aware of any comparison of this kind in the literature.

The outline of the paper is as follows: 
We start with  introducing INNs with an appropriate loss function to sample posterior distributions of inverse problems
in Section \ref{sec:model}.  Here we follow the lines of  an earlier version of \cite{kruse2020hint}.
In particular, the likelihood function has to be adapted to our noise model with different variances in each component of the measurement for the application at hand.
Then, in Section  \ref{sec:application}, we describe the forward model in GIXRF in its experimental, 
numerical and surrogate setting.
The comparison of INN with MCMC posterior sampling in done in Section \ref{sec:numerics}.
Finally, conclusions are drawn and topics of further research are addressed in Section \ref{sec:conclusions}.

%% file: src/math.tex
\section{Posterior Reconstruction by INNs} \label{sec:model}
In this section, we explain, based on \cite{kruse2020hint}, how the posterior
of an inverse problems can be analyzed using INNs.
In the following, products, quotients and exponentials of vectors are meant componentwise.
Denote by $p_x$ the density function of a distribution $P_X$ of a random variable 
$X\colon\Omega\to \mathbb{R}^d$. 
Further, let $p_{x|y} (\cdot|y)$ be the density function of the
conditional distribution  $P_{X|Y=y}$ of $X$ given the value of a random variable
$Y\colon\Omega\to\mathbb{R}^n$ at $Y = y \in \R^n$.  
We suppose that we have a differentiable forward model 
$f: \R^d \rightarrow \R^n$. 
In our applications the forward model will be given by the GIXRF method.
We assume that the measurements $y$ are corrupted by additive Gaussian noise 
$\mathcal{N}(0,b^2\diag(w^2))$, where the (positive) weight vector $w\in\R^n$ accounts for the different
scales of the measurement components. The factor $b^2 >0$ models the intensity
of the noise. 
In other words, 
$$
y = f(x) + \eta, 
$$
where $\eta$ is a realization of a $\mathcal{N}(0,b^2\diag(w^2))$ distributed random variable.
Then the sampling density reads as
\begin{align} 
  \label{eqn:likelihood}
  p_{y|x}(y|x) 
  = \frac{1}{(2\pi)^\frac{n}{2} b^n (\prod_{i=1}^n w_i)} 
  \exp\left( -\frac{1}{2} \left\Vert \frac{y-f(x)}{bw}\right\Vert^2 \right).
\end{align}
Given a measurement 
$y$
from the forward model,
we are interested in the inverse problem posterior distribution $P_{X|Y=y}$.
By Bayes' formula the inverse problem posterior density can be rewritten as
\begin{align*}
  p_{x|y}
  = \frac{p_{y|x} p_x}{\int_{\mathbb{R}^d} p_{y|x}(y|x)p_x(x)\ \mathrm{d} x}
  \propto p_{y|x}\, p_x.
\end{align*}
Let $p_\xi$ be the density function of an easy to sample distribution $P_\Xi$
of a random variable $\Xi\colon\Omega\to \mathbb{R}^d$. 
Following for example \cite{Marzouk_2016}, we want to find a differentiable and invertible map
$\mathcal{T}\colon \mathbb{R}^d\to \mathbb{R}^d$, 
such that $\mathcal{T}$
pushes $P_\Xi$ to $P_{X|Y = y}$, i.e., 
\begin{equation} \label{eq:push_forw}
P_{X|Y = y} = \mathcal{T}_\# P_\Xi \coloneqq P_\Xi \circ \mathcal{T}^{-1}.
\end{equation} 
Recall that $\mathcal{T}_\#p_\xi = p_\xi\circ \mathcal{T}^{-1} \vert \det
\nabla T^{-1}\vert$ for the corresponding density functions, where $\nabla
T^{-1}$ denotes the Jacobian of $T^{-1}$.

Once $\mathcal{T}$ is learned for some measurement $y$,  
sampling from the posterior $P_{X|Y = y}$ can be
approximately done by evaluating $\mathcal{T}$ at samples from the reference distribution
$P_\Xi$ (see also \cite{Marzouk_2016}).  
Since it
is in general hard or even impossible to find the analytical map $\mathcal{T}$, we aim to
approximate $\mathcal{T}$ by an invertible neural network 
$T = T(\bullet;\theta)\colon \mathbb{R}^d\to \mathbb{R}^d$ with network parameters
$\theta$.  
In this paper, we use a variation of the INN proposed in
\cite{dinh2017density}, see \cite{ardizzone2019analyzing}. 
More precisely, $T$ is a composition 
\begin{align} \label{eq:net}
  T =  T_L \circ P_L \circ \dots \circ  T_1 \circ P_1,
\end{align}
where $P_\ell$ are permutation matrices 
and $T_\ell$ are invertible mappings of
the form
\begin{equation}
  \label{eq:DefBlock}
  T_\ell(\xi_1,\xi_2) 
  = (x_1,x_2) 
  \coloneqq \left(\xi_1 \, \mathrm{e}^{s_{\ell,2}(\xi_2)} + t_{\ell,2}(\xi_2),\, 
          \xi_2 \, \mathrm{e}^{s_{\ell,1}(x_1)}   + t_{\ell,1}(x_1)   \right)
\end{equation}
for some splitting $(\xi_1,\xi_2) \in \mathbb{R}^{d}$ with $\xi_i \in \mathbb R^{d_i}$, $i=1,2$.  Here
$s_{\ell,2}, t_{\ell,2}: \mathbb R ^{d_2} \rightarrow \mathbb R ^{d_1}$ 
and 
$s_{\ell,1}, t_{\ell,1}: \mathbb R ^{d_1} \rightarrow \mathbb R ^{d_2}$ 
are ordinary feed-forward neural networks. 
The parameters $\theta$ of
$T(\cdot;\theta)$ are specified by the parameters of these subnetworks.
The inverse of the layers $T_\ell$ is analytically given by
\begin{equation}\label{eq:DefInvBlock}
    T_\ell^{-1}(x_1,x_2) 
    = (\xi_1,\xi_2) 
		\coloneqq \left( \big(x_1 - t_{\ell,2}(\xi_2) \big) \,\mathrm{e}^{-s_{\ell,2}(\xi_2)},\, 
             \big(x_2 - t_{\ell,1}(x_1)   \big) \,\mathrm{e}^{-s_{\ell,1}(x_1)} \right)
\end{equation}
and does not require an inversion of the feed-forward subnetworks. Hence the whole
map $T$ is invertible and allows for a fast evaluation of both forward and
inverse map.

In order to learn the INN, we utilize the Kullback-Leibler divergence as a measure of distance
between two distribution as loss function
\begin{align} 
  \label{eqn:loss_1}
  L(\theta) 
  \coloneqq \operatorname{KL} (T_{\#} p_\xi,\, p_{x|y}) 
  = \int_{\mathbb{R}^d} T_{\#} p_\xi  \log \left( \frac{T_{\#} p_\xi}{p_{x|y}} \right) \,\mathrm{d} x.
\end{align}
Minimizing the loss function $L$ by e.g. a standard stochastic
gradient descent algorithm requires the computation of the gradient of $L$.
To ensure that this is feasible, we rewrite the loss $L$ in the following way.

\begin{Proposition}
  Let $T = T(\bullet; \theta)\colon \mathbb{R}^d \rightarrow \mathbb{R}^d$ be a 
  Lebesgue measurable diffeomorphism parameterized by $\theta$.  Then, up to an additive constant, \eqref{eqn:loss_1} can be written 
  as
  \begin{align} 
    \label{eqn:loss_2}
    L(\theta) 
    = - \mathbb{E}_\xi \Bigl[ \log p_{y|x}\bigl(y|T(\xi)\bigr)
    + \log p_x\bigl( T(\xi)\bigr) + \log \vert \det \nabla T(\xi)\vert \Bigr].
  \end{align}
  For the Gaussian likelihood \eqref{eqn:likelihood}, this simplifies to
  \begin{align} 
    \label{eqn:loss_3}
    L(\theta) 
    &= \mathbb{E}_\xi \Bigl[\frac{1}{2b^2} \left\Vert \frac{y-(f\circ T)(\xi)}{w} \right\Vert^2
    - \log p_x\bigl( T(\xi)\bigr) - \log \vert \det \nabla T(\xi)\vert \Bigr].
  \end{align}
\end{Proposition}

\begin{proof}
  By definition of the push-forward density and the transformation formula
  \cite[Theorem 7.26]{Rudin}, we may rearrange 
  \begin{align*}
    &\operatorname{KL}\bigl(T_{\#}p_\xi, p_{x|y} \bigr)\\[1ex]
    &\qquad= \int_{\mathbb{R}^d} (p_\xi\circ T^{-1})(x)\, 
    \vert\det \nabla T^{-1}(x)\vert\, 
    \log\left( \frac{(p_\xi\circ T^{-1})(x) \vert\det \nabla T^{-1}(x)\vert}{p_{x|y}(x|y)}\right)\,\mathrm{d} x \\[1ex]
    &\qquad= \int_{\mathbb R^d} p_\xi(\xi)\, 
    \log\left(\frac{p_\xi(\xi)}{p_{x|y}(T(\xi)|y) \, \vert\det\nabla T(\xi)\vert}\right) \,\mathrm{d}\xi \\[1ex]
    &\qquad= \operatorname{KL}\bigl(p_\xi, T^{-1}_{\#}p_{x|y} \bigr)\\[1ex]
    &\qquad= -\mathbb{E}_{\xi}\left[
    \log p_{x|y}(T(\xi)| y ) + \log \vert\det\nabla T(\xi)\vert\right]  
    + \mathbb{E}_{\xi}[\log p_\xi].
  \end{align*}
  Now Bayes' formula yields
  \begin{align*}
    \log p_{x|y}(T(\xi)|y)
    = \log \Bigl( p_{y|x}(y|T(\xi))\, p_x(T(\xi))\Bigr) - \log p_y(y).
  \end{align*}
  Ignoring the constant terms since they are irrelevant for the minimization of
  $L$, we obtain \eqref{eqn:loss_2}. The rest of the assertion follows
  by~\eqref{eqn:likelihood}. 
  \hfill $\Box$
\end{proof}

The different terms on the right-hand side of \eqref{eqn:loss_2}, resp.
\eqref{eqn:loss_3} are interpretable: The first term forces the samples pushed
through $T$ to have the correct forward mapping, the second assures that the
samples are pushed to the support of the prior distribution and the last term
employs a counteracting force to the first. To see this note that the first
term is minimized if $T$ pushes $p_\xi$ to a delta distribution, whereas the
log determinant term would be unbounded in that case.  Hence there is an
equilibrium between those terms directly influenced by the error parameter $b$,
i.e. as $b$ tends to zero, the push-forward tends to a delta distribution.

The computation of the gradient of the empirical loss function $L$ corresponding to \eqref{eqn:loss_3}
requires besides standard differentiations of elementary functions and of the network $T$,
the differentiation of i) the forward model $f$ within the chain rule of $f \circ T$, ii) of $p_x$, and 
iii) of $|\det \nabla T|$. This can be done by the following observations:
\begin{enumerate}[label=\roman*)]
  \item In the next section, we describe how a feed-forward neural network can
    be learned to approximate the forward mapping in GIXRF.  Then this network
    will serve as forward operator and its gradient can be computed by standard
    backpropagation.
  \item The prior density $p_x$ has to be known.  In our application, we can
    assume that the geometric parameters $x$ are uniformly distributed in a
    compact set which is specified for each component of $x \in \mathbb R^d$ in
    the numerical section.  This has the consequence that the term $\log p_x$
    is constant within the support of $p_x$ and is not defined outside.  Therefore, we
    impose an additional boundary loss  that penalizes samples out of the
    support of the prior, more precisely, if the prior is supported in $[s_1, t_1]
    \times ... \times [s_d, t_d] $, then we use
		\begin{align*}
      L_\mathrm{bd}(x)
      = \lambda_\mathrm{bd}\sum_{i=1}^d\bigl(
      \operatorname{ReLU}(x_i-t_i)+\operatorname{ReLU}(s_i-x_i)
      \bigr),
      \qquad\lambda_\mathrm{bd} > 0.
    \end{align*}
    Note that the non-differentiable $\operatorname{ReLU}$ function at zero can
    be replaced by various smoothed variants.
  \item For general networks, $\log |\det \nabla T|$ in the loss function is
    hard to compute and is moreover either not differentiable or has a huge
    Lipschitz constant.  However, it becomes simple for INNs due to their
    special structure.  Since $T_\ell = T_{2,\ell} \circ T_{1,\ell}$ with $
    T_{1,\ell}(\xi_1,\xi_2) = (x_1,\xi_2) \coloneqq
    \left(\xi_1\mathrm{e}^{s_{\ell,2}(\xi_2)} + t_{\ell,2} (\xi_2), \xi_2
    \right), $ and $ T_{2,\ell}(x_1,\xi_2) = (x_1,x_2) \coloneqq \left(x_1,
    \xi_2\mathrm{e}^{s_{\ell,1}(x_1)} + t_{\ell,1} (x_1) \right) $ we have 
    \begin{align*}
		  \nabla T_{1,\ell}(\xi_1,\xi_2) = 
		  \begin{pmatrix}
		    \mathrm{diag} \left( \mathrm{e}^{s_{\ell,2}(\xi_2)} \right) & \mathrm{diag} \left( \nabla_{\xi_2} \left( \xi_1 \mathrm{e}^{s_{\ell,2}(\xi_2)} + t_{\ell,2} (\xi_2) \right) \right)\\
		    0 & I_{d_2}
		  \end{pmatrix}
    \end{align*}
    so that $ \det \nabla T_{1,\ell}(\xi_1,\xi_2) =  \prod_{k=1}^{d_1}
    \mathrm{e}^{\left( s_{\ell,2}(\xi_2)\right)_k} $ and similarly for $\nabla
    T_{2,\ell}$.  Applying the chain rule in \eqref{eq:net}, noting that the
    Jacobian of $P_\ell$ is just $P_\ell^\tT$ with $\det P_\ell^\tT=1 $, and that
    $\det (A B) = \det(A) \det(B)$, we conclude 
    \begin{align*}
		  \log( |\det \left(\nabla T(\xi) \right)|)
		  = \sum_{\ell = 1}^L \left( \operatorname{sum}\left(s_{\ell,2} \left( (P_\ell \xi^{\ell} )_2 \right)\right) 
		  + \operatorname{sum}\left(s_{\ell,1}\left( (T_{1,\ell} P_\ell \xi^{\ell} )_1 \right) \right)\right),
    \end{align*}
    where $\operatorname{sum}$ denotes the sum of the components of the respective vector,
    $\xi^{1} \coloneqq \xi$ and $\xi^{\ell} = T_{\ell-1} P_{\ell-1}
    \xi^{\ell-1}$, $\ell = 2,\ldots,L$.
\end{enumerate}

%% file: src/application.tex
\section{Forward Model from GIXRF} \label{sec:application}
In this section, we consider a silicon nitride ($\mathrm{Si_3N_4}$)
lamellar grating on a silicon substrate. The grating  oxidized in a natural fashion resulting in a thin $\mathrm{SiO_2}$ layer.
A cross-section of the lamellar grating is shown in Fig. \ref{fig:grating}, left.
It can be characterized by seven parameters $x \in \mathbb R^d$, $d=7$, namely
the height ($h$) 
and middle-width (cd) of the line, 
the sidewall angle (swa),
the thickness of the covering oxide layer ($t_t$), 
the thickness of the etch offset of the covering oxide layer beside the lamella ($t_g$)
and additional layers on the substrate ($t_s$, $t_b$).

\begin{figure}[!htb]
  \begin{center}
    \includegraphics[height=0.23\textheight]{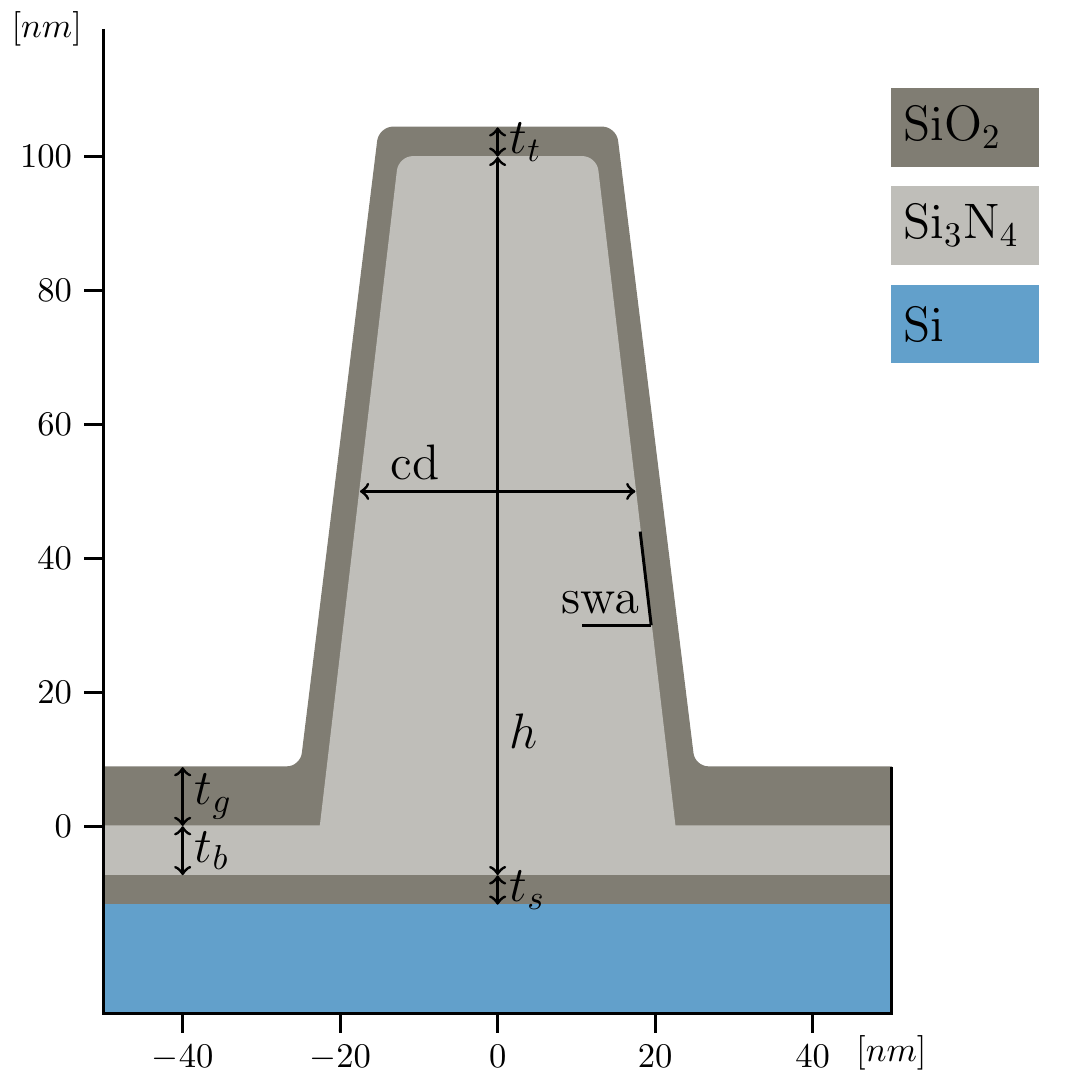}
		\includegraphics[height=0.23\textheight]{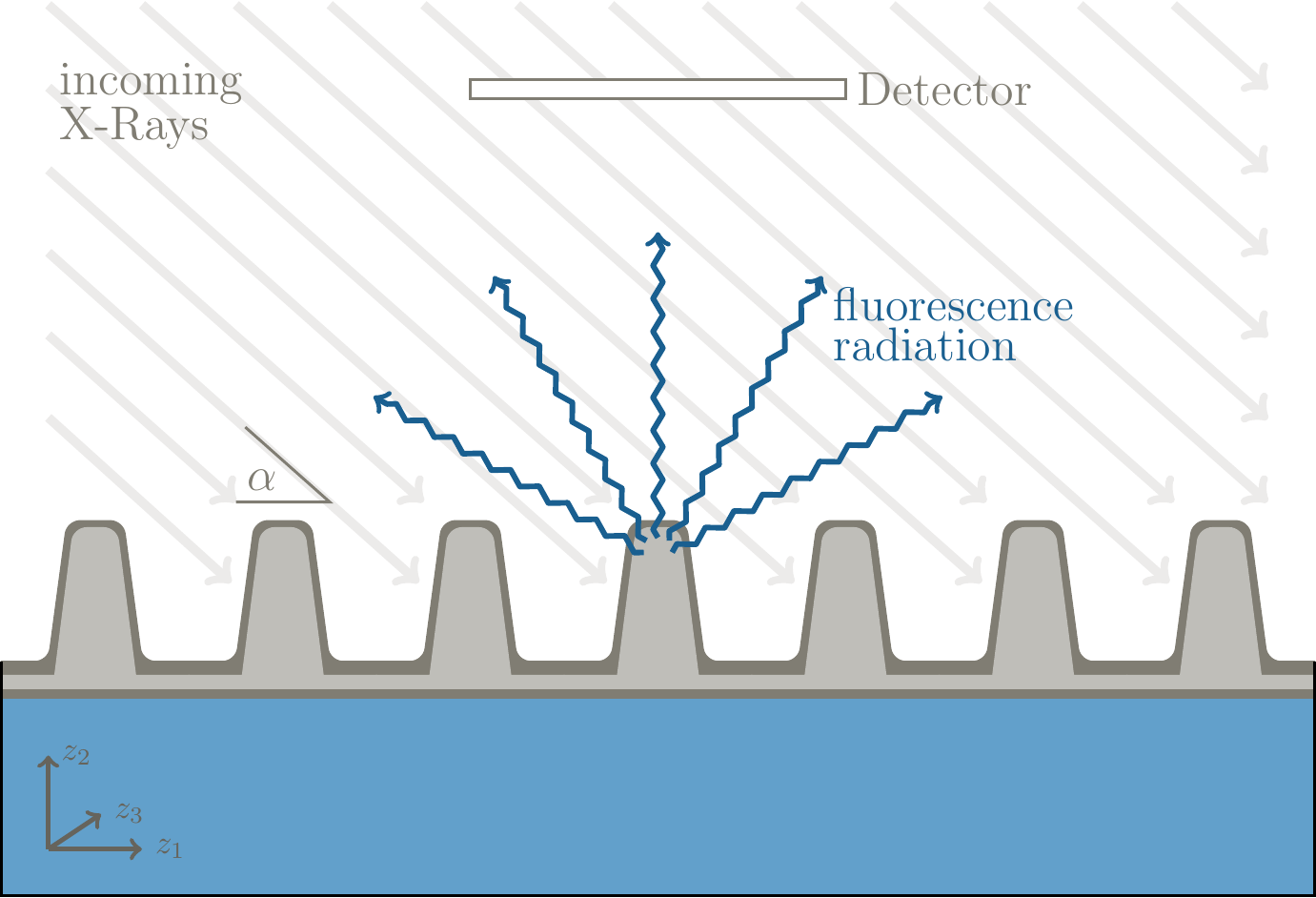}
  \end{center}
  \caption{ Left: Cross-section of one grating line with characterizing parameters.  
    Right: Cross-section of the grating with incoming X-rays (typically in $z_3$ direction) and emitted
    fluorescence radiation.
	}
  \label{fig:grating}
\end{figure}

To determine the parameters, we want to apply the GIXRF technique recently established to find
the geometry parameters of nanostructures and its atomic composition \cite{andrle2019grazing,honicke2020grazing}.

\paragraph{Experimental Setting.}
In GIXRF, the angles $\alpha_i$, $i=1,\ldots,n$ between 
an incident monochromatic X-ray beam and the sample surface is varied around a critical angle for total external reflection. Depending on the local field intensity of the X-ray radiation, 
atoms are ionized and are emitting a fluorescence radiation.   The resulting detected fluorescence radiation $(F(\alpha_i;x))_{i=1}^n$ is characteristic for the atom type. 

Besides direct experimental measurements of the fluorescence $F$,
its mathematical modeling at each angle $\alpha$
can be done in two steps, namely by computing the intensity of the local electromagnetic field $E$ arising from the incident wave (X-ray) with angle $\alpha$  
and then to use its modulus to obtain the fluorescence value $F(\alpha,x)$.
The propagation of $E$  is in general described by Maxwell’s equations and simplifies for our specific geometry  to the partial differential equation 
\begin{equation}
  \label{eqn:Maxwell}
  \operatorname{\nabla_z} \cdot \left(\mu(z;x)^{-1} 
  \operatorname{\nabla_z} E(z) \right)
  - \omega^2 \varepsilon(z;x) \, E(z) = 0, \quad z \in \mathbb R^2.
\end{equation}
Here $\omega$ is the frequency of the incident plane wave, 
$\varepsilon$ and $\mu$ are the permittivity and permeability depending on the grating parameters $x$, resp.
From the 2D distribution of 
$E$ in \eqref{eqn:Maxwell}, more precisely from the resulting field intensities
$|E|$, the fluorescence radiation  $F(\alpha;x)$ at the detector can be calculated 
by an extension of the Sherman equation \cite{soltwisch2018element}.
This computation requires just the appropriately scaled summation of the values of $|E|^2$ on the FEM mesh of the Maxwell solver used.
The really time consuming part is the numerical computation of $E$ for each angle $\alpha_i$, $i=1,\ldots,n$
by solving the PDE \eqref{eqn:Maxwell}.

\paragraph{Numerical treatment.}
In order to compute 
$E$ given by \eqref{eqn:Maxwell} with appropriate boundary conditions,
we employ the finite element method (FEM)  implemented in the
JCMsuite software package to discretize and solve the corresponding scattering
problem on a bounded computational unit cell in the weak formulation as described in \cite{JCMsuite}.  
This formulation yields a splitting of the complete $\mathbb{R}^2$ into an interior domain hosting the total field and an exterior domain, where only the purely outward radiating scattered field is present.  At the boundaries, Bloch-periodic boundary conditions are applied in the periodic lateral direction and an adaptive perfectly matched layer (PML) method was used to realize transparent boundary conditions. 


\paragraph{Surrogate NN model}
The evaluation of the fluorescence intensity $F$ for a single realization of
the parameters $x$ involves solving  \eqref{eqn:Maxwell} for each angle 
$\alpha_i$, $i=1,\dots,n$. Since this is very time consuming,  we learn instead a simple
feed-forward neural network with one hidden layer with $256$ nodes and
ReLU activation as surrogate 
$f\colon \mathbb{R}^d\to\mathbb{R}^n$ of $F$ such that the $L^2$ error between 
$\left(f_i(x) \right)_{i=1}^n$ and $\left(F(\alpha_i;x) \right)_{i=1}^n$ becomes minimal.
The network was trained 
on roughly $10^4$ sample pairs $(F,x)$ 
which were numerically generated as described above in a time consuming procedure.
The $L^2$-error of the surrogate, evaluated on a separate test set containing about
$10^3$ sample pairs was smaller than $2\cdot 10^{-3}$. This is sufficient for the
application, since we have measurement noise on the data of at least one order
of magnitude larger for both the experimental data and the synthetic
study. Hence we neglect the approximation error of the FE model and the
surrogate further on.

%% file: src/results.tex
\section{Numerical Results} \label{sec:numerics}
In this section, we solve the statistical inverse problem of GIXRF using the Bayesian approach with an INN and the MCMC method based on our surrogate NN forward model for virtual and experimental data. 
Note that obtaining training data for the forwardNN took multiple
days of computation on a compute server with 120 CPUs.
We compare the resulting posterior distributions and the computational performance. 
To the best of our knowledge, this was not done in the literature so far
for any forward model. 

The fluorescence intensities for the silicon nitride layer of the lamellar 
grating depicted in Fig. \ref{fig:grating} were measured for $n=178$ different
incidence angles $\alpha_i$ ranging from $0.8^\circ$ to $14.75^\circ$.  The
seven parameters of the grating were considered to be uniformly distributed
according to the domains listed in Tab. \ref{tab:param_doms}. 

\begin{table}
  \setlength\extrarowheight{5pt}
  \centering
  \begin{tabular}{p{0.20\textwidth}p{0.14\textwidth}p{0.12\textwidth}p{0.12\textwidth}p{0.08\textwidth}p{0.08\textwidth}p{0.1\textwidth}p{0.07\textwidth}}
    {\bf Parameter:} & $h$ & $cd$ & $\mathrm{swa}$ & $t_t$ & $t_b$ & $t_g$ & $t_s$ \\[1ex]
    \hline
    {\bf Domain:} & $[85,\,100]$ & $[45,\,55]$ & $[76,\,88]$ & $[2,\,4]$ & $[0,\,5]$ & $[2,\,10]$ & $[0.1,\,3]$ \\[4ex]
  \end{tabular}
  \caption{Domains of the different parameters. Units are given in 
           $[\mathrm{nm}]$ for all parameters except the sidewall angle ($\mathrm{swa}$), which
           is given in $[^\circ]$.}
  \label{tab:param_doms}
\end{table}
To gain maximal performance of the MCMC method, we utilize an affine invariant
ensemble sampler for the Markov-Chain Monte Carlo algorithm \cite{emcee}.  This
allows parallel computations with multiple Markov chains and reduces the number
of method specific free parameters for the MCMC steps.  The
error parameter $b$ is usually a priori unknown and is thus
subject to expert knowledge. However, MCMC algorithms can introduce those as
additional posterior hyperparameters for reconstruction by a slight
modification of the prior.  Define $\tilde x \coloneqq (x,b)\sim P_{\tilde x}$,
where $P_{\tilde x}$ is given by the density $ p_{(x,b)}= p_x\,  p_b $ for
uniform $p_b$.  Using this error model in the likelihood in
\eqref{eqn:likelihood}, we obtain $p_{y|(x,b)}$.  Then the MCMC algorithm applied
to $p_{(x,b) | y}$  yields the distribution of the parameter $b$ as well. 

To approximate the INN that pushes the example density forward to the posterior
one, we learn an INN with $L=10$ layers.  Each subnetwork of each layer is
chosen as a two layer ReLU feedforward network with $256$ hidden nodes in each
layer.  The network is trained on the empirical counterpart of the loss
\eqref{eqn:loss_3} by sampling from a standard Gaussian distribution using an
adaptive moment estimation optimization (Adam) algorithm, see \cite{kingmaadam}.  
We trained for 80
epochs, an epoch consists of 40 parameter updates with a batch-size of 200. The
learning rate is lowered by a factor of 0.1 every 20 epochs.  The INN model is
built and trained with the freely available FrEIA software package
\footnote{https://github.com/VLL-HD/FrEIA}.

The time that is required to obtain $2\cdot 10^{4}$ posterior samples via the MCMC algorithm varies between $1.5$ and $3.5$ hours on a standard Laptop. In
comparison the training of the INN takes less than $20$ minutes, and  $2\cdot 10^{4}$ independent posterior samples are generated in less than one
second.  

\subsection{Synthetic Data}

As a first application, we perform a virtual experiment for the GIXRF to
obtain a problem with known ground truth. To approach that, we pick a
pair  $(x_\mathrm{true},f(x_\mathrm{true}))$  and vary
$b\in\{10^{-2},3\cdot10^{-2},10^{-1}\}$. 
Using these values of $b$ we obtain a synthetic noisy measurement according to
$y_\mathrm{meas}=y_\mathrm{true}+\varepsilon$
for various realizations of  $\varepsilon$ of $\mathcal{N}(0,b^2
\operatorname{diag}(y_\mathrm{true}^2))$. 
For the computation of the INN, we
set $b$ according to the true value, whereas MCMC is able to estimate a
distribution of the parameter $b$ for given $y_\mathrm{meas}$. For the application
of both MCMC and INN, we set $y = y_\mathrm{meas}$ and  $w =
y_\mathrm{meas}$ (note that we regard $w = y_\mathrm{meas}$ as a constant).
Fig.~\ref{fig:comp_methods} displays the one dimensional
marginals of the posterior for both the MCMC and the INN approach alongside
the ground truth for three different values of $b$.

\begin{figure}
  \begin{center}
    \begin{subfigure}[t]{0.32\textwidth}
      \includegraphics[width=\textwidth]{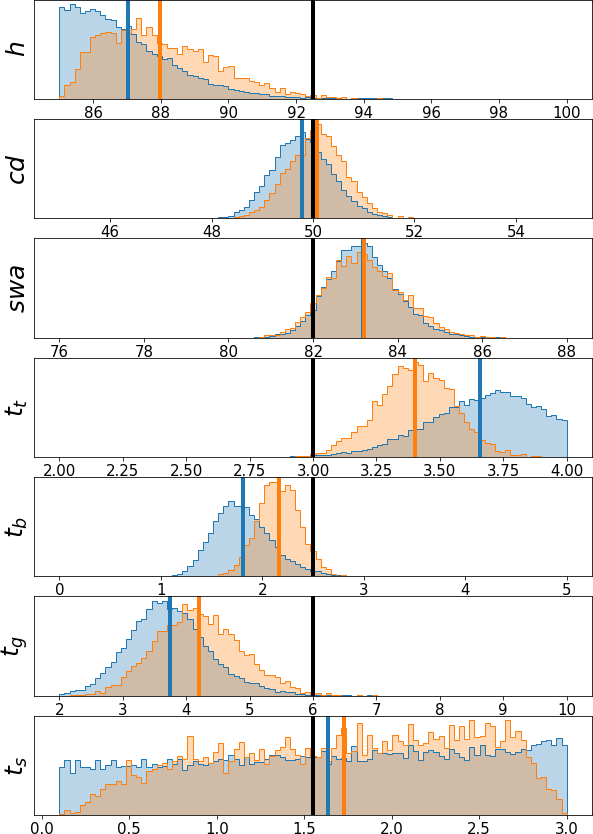}
      \caption{$b = 0.1$}
    \end{subfigure}
    \begin{subfigure}[t]{0.32\textwidth}
      \includegraphics[width=\textwidth]{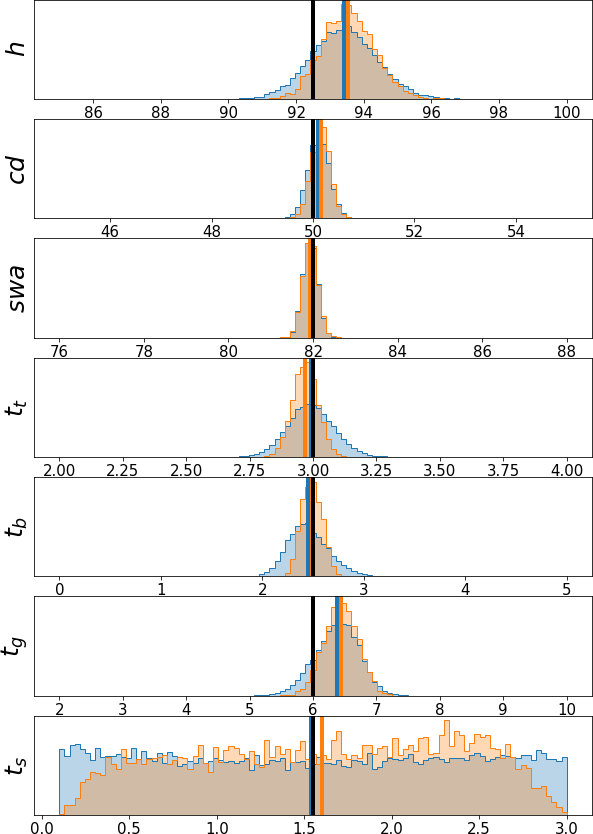}
      \caption{$b = 0.03$}
    \end{subfigure}
    \begin{subfigure}[t]{0.32\textwidth}
      \includegraphics[width=\textwidth]{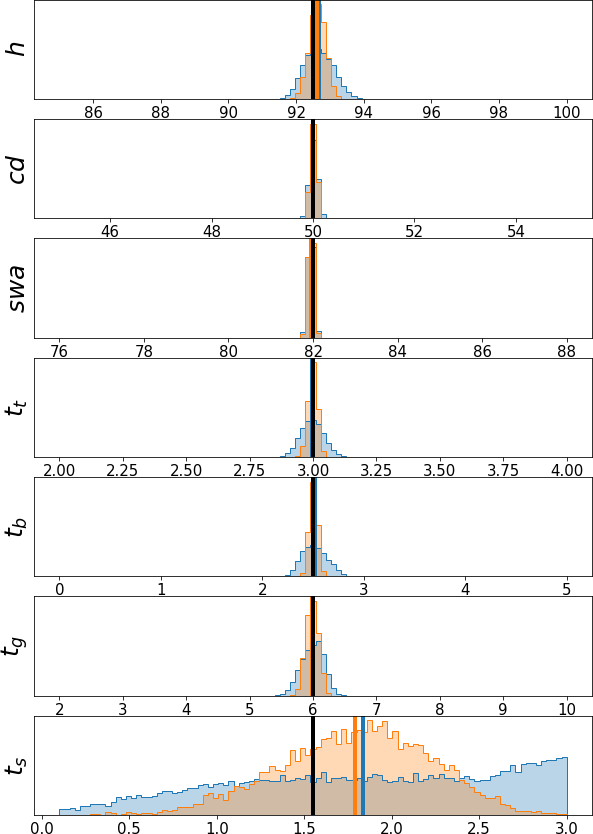}
      \caption{$b = 0.01$}
    \end{subfigure}
  \end{center}
  \caption{One dimensional marginals of posterior samples for different values of
           $b$ generated by MCMC (blue) and INN (orange) together with their means as solid lines.
					The values of the ground truth $x_i$, $i=1,\ldots,7$ are displayed as solid black lines.}
  \label{fig:comp_methods}
\end{figure}

First one sees that the width of the marginals
decreases as $b$ gets smaller. Comparing the INN and MCMC marginals shows an
almost identical shape and support for most of the parameters, where the
uncertainties obtained by the INN tend to be a bit smaller. The posterior means
of the MCMC and INN approach are in proximity of each other relative to the
domain size, in particular in the cases, where $b$ is smaller. It is remarkable
that although the reconstruction in $b = 0.1$ is far-off from the ground truth,
both methods agree in their estimate.  This can be explained by the large magnitude of
noise. The noise realization for $b = 0.1$ changes the measurements drastically
such that a different parameter configuration becomes more likely. Furthermore, both methods identify the last parameter to be the least sensitive.

\subsection{Experimental Data}

\begin{figure}
  \begin{center}
    \begin{subfigure}[b]{0.7\textwidth}
      \includegraphics[width=\textwidth]{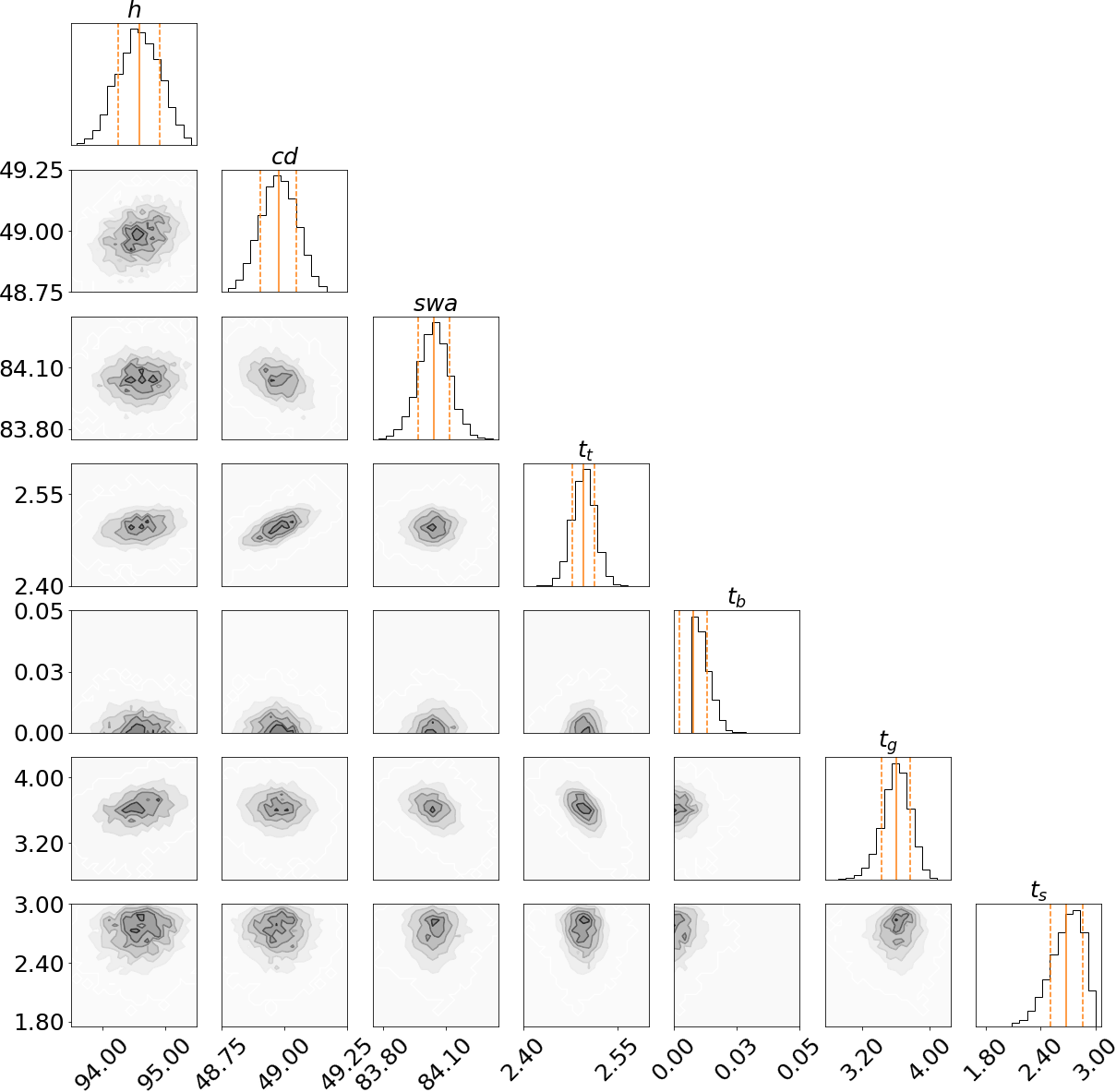}
      \caption{INN.}
    \end{subfigure}
    \begin{subfigure}[b]{0.7\textwidth}
      \includegraphics[width=\textwidth]{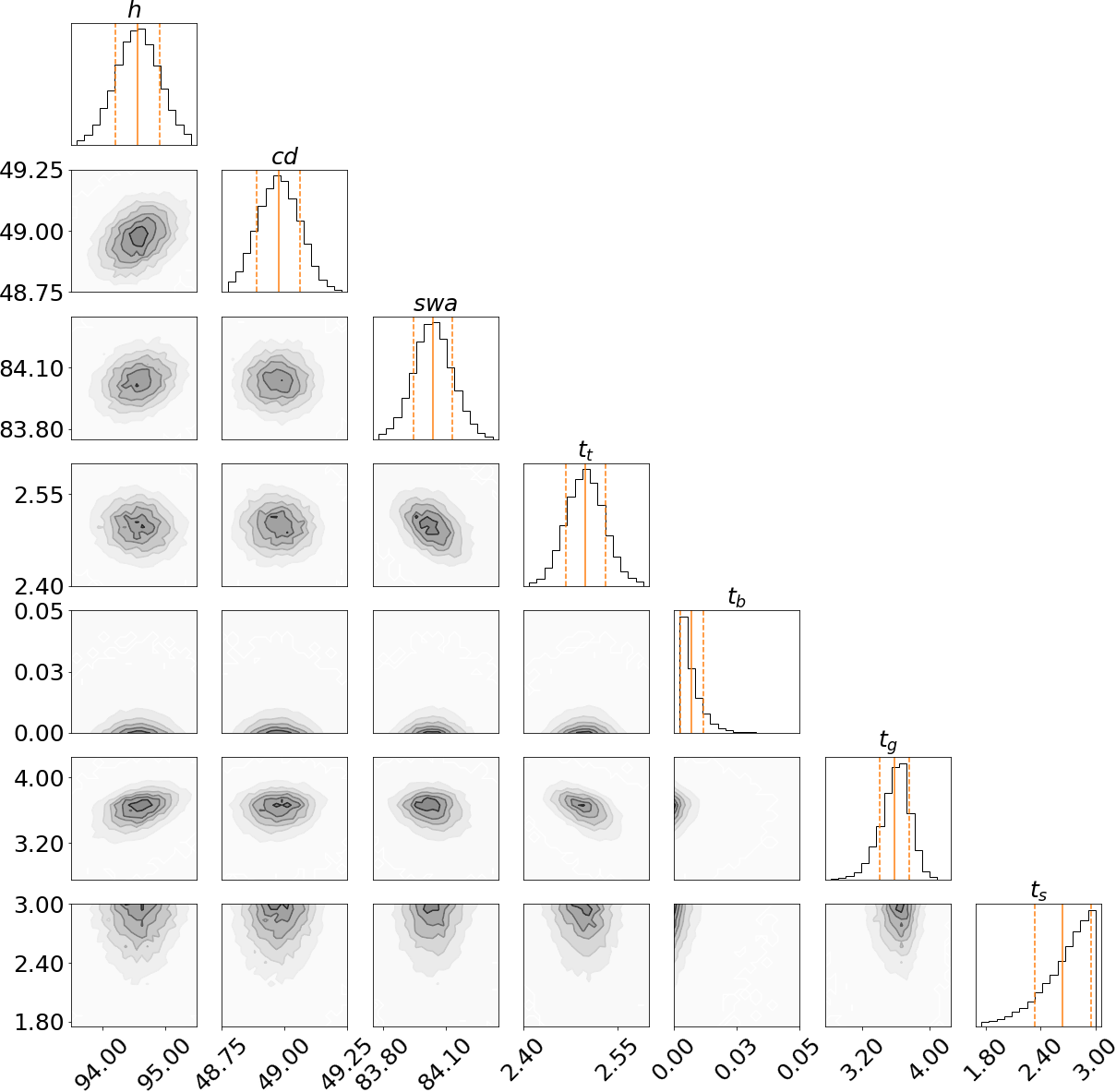}
      \caption{MCMC.}
    \end{subfigure}
  \end{center}
  \caption{Comparison of 2D densities of the posterior for $y_\mathrm{meas}$ calculated via the different methods. The straight vertical line depicts the mean, the dashed ones the standard deviation.}
  \label{fig:comp_for_exp_data}
\end{figure}

In the next experiment,  we use fluorescence measurements $y_{\mathrm{meas}} = \left(
F(\alpha_i;x) \right)_{i=1}^n$ obtained from an GIXRF experiment.  Here
neither exact values of the parameters $x$ nor of the noise level $b$ are
available.  In the noise model \eqref{eqn:likelihood} we use $w =
y_{\mathrm{meas}}$.  To train the INN we set the value $b$ to the mean $0.02$ of the
reconstructed $b$ obtained by the MCMC algorithm.
Fig.~\ref{fig:comp_for_exp_data} depicts the one and two dimensional marginals
of the posterior for both the MCMC and the INN calculations on a smaller
subset of the parameter domain. The means and supports of the posteriors agree well.   Fig.~\ref{fig:comp_for_exp_data} also shows
that the posterior is mostly a sharp Gaussian. 
The $Si_3N_4$ height etch offset $t_b$ is known to be close to zero for
the real grating, which explains the non-Gaussian accumulation at the boundary of the interval. Similar holds true for the height of the lower
$SiO_2$ layer.

\begin{figure}
  \begin{center}
    \includegraphics[width=\textwidth]{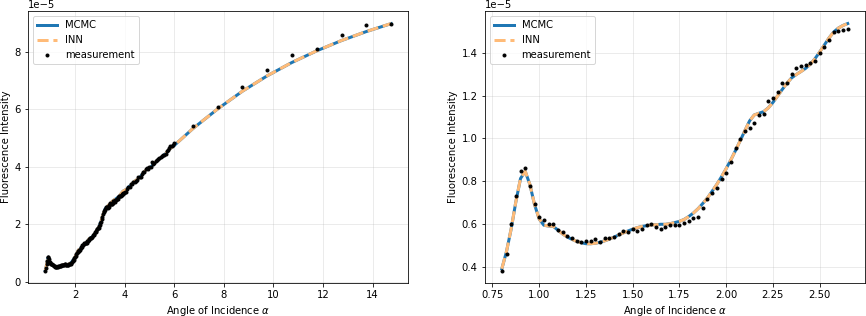}
  \end{center}
  \caption{Forward fit of posterior means obtained by INN (orange) and MCMC
           (blue). Measurement data are represented by black points. The right image is a zoom 
					into the left one.}
  \label{fig:forward_fit}
\end{figure}

Finally, we apply the forward model $f$ to the componentwise mean of the
sampled posterior values $x \in \mathbb R^7$.  The results are shown in
Fig.~\ref{fig:forward_fit}.  The resulting $f(x)$ fits both for the MCMC and for the INN quite well with the experimental measurement $y_{\mathrm{meas}}$.  The
differences in the mean and shape of the posteriors seem to be small and may be caused by small numerical or model errors. 

%% file: src/conclusion.tex
\section{Conclusions}\label{sec:conclusions}

We have shown that INNs provide comparable results to established MCMC methods in sampling posterior distributions of parameters in GIXRF, but outperform them in terms of computational time. Moreover INNs are more flexible for different applications and are expected to perform well in high dimensions. Since INNs are often used in very
high-dimensional problems, such as generative modeling of images, this suggests that our approach could scale well to more challenging
high-dimensional problems. 
So far, we considered only a single measurement $y_{\mathrm{meas}}$, but extensions to capture different measurements seem to be feasible. 
Furthermore, we will figure out, whether the noise parameter $b$ can be learned as well in the INN framework. Finally, we intend to have a closer look at multimodal distributions as done, e.g., in \cite{HN2020}.